\documentclass{colt2016} 

\title[Highly-Smooth Optimization]{Highly-Smooth Zero-th Order Online Optimization}
\usepackage{times}

\newcommand{\BEAS}{\begin{eqnarray*}}
\newcommand{\EEAS}{\end{eqnarray*}}
\newcommand{\BEA}{\begin{eqnarray}}
\newcommand{\EEA}{\end{eqnarray}}
\newcommand{\BEQ}{\begin{equation}}
\newcommand{\EEQ}{\end{equation}}
\newcommand{\BIT}{\begin{itemize}}
\newcommand{\EIT}{\end{itemize}}
\newcommand{\BNUM}{\begin{enumerate}}
\newcommand{\ENUM}{\end{enumerate}}
\newcommand{\BA}{\begin{array}}
\newcommand{\EA}{\end{array}}

\newcommand{\idm}{I}
\newcommand{\rb}{\mathbb{R}}

\newcommand{\eq}[1]{Eq.~(\ref{eq:#1})}

\def \E{{\mathbb E}}

\def \N{{\mathbb N}}

\def \E{{\mathbb E}}

\def \F{{\mathcal F}}

 \coltauthor{\Name{Francis Bach} \Email{francis.bach@ens.fr}\\
 \addr INRIA \&
D\'epartement d'Informatique
de l'Ecole Normale Sup\'erieure 
  \AND
 \Name{Vianney Perchet} \Email{vianney.perchet@normalesup.org}\\
 \addr CREST - ENSAE
 }

\begin{document}

\maketitle

\begin{abstract}
The minimization of convex functions which are only available through partial and noisy information is a key methodological problem in many disciplines. In this paper we consider  convex optimization with noisy zero-th order information, that is noisy function evaluations at any desired point. We focus on problems with high degrees of smoothness, such as  logistic regression. We show that as opposed to gradient-based algorithms, high-order smoothness may be used to improve estimation rates, with a precise dependence  of our upper-bounds on the degree of smoothness. In particular, we show that for infinitely differentiable functions, we recover the same dependence on sample size as gradient-based algorithms, with an extra dimension-dependent factor. This is done for both convex and strongly-convex functions, with finite horizon and anytime algorithms. Finally, we also recover similar results in the online optimization setting.\end{abstract}


\section{Introduction}

The minimization of convex functions which are only available through partial and noisy information is a key methodological problem in many
disciplines. When first-order information, such as gradients, is available, many algorithms and analysis have been proposed \citep[see, e.g.,][and references therein]{shalev2011online}, taking the form of stochastic gradient descent~\citep{robbins_stochastic_1951},  online mirror descent~\citep{lan2012validation},  dual averaging~\citep{xiao2010dual} or even variants of ellipsoid methods \citep{nemirovsky1983problem,Rakhlin2013}. Strong convexity has emerged as an important property characterizing the performance of these algorithms, with optimal convergence rates of $O(1/n)$ after $n$ iterations for strongly-convex problems, and only $O(1/\sqrt{n})$ for convex problems. 

However, smoothness can typically only improve constants~\citep{lan2010optimal}, with the stochastic part of the generalization performance having the same scalings than in the non-smooth case. Apart for quadratic functions or logistic regression where the rates may be improved~\citep{newsto,shamir2013complexity,HazanKorenLevy14}, the boundedness of high-order derivatives is typically not advantageous.

In this paper, we consider situations where only noisy function values are available, originating from derivative-free optimization~\citep{spall2005introduction} and with  increased received attention~\citep[see, e.g.,][and references therein]{bubeck2012regret}. This  is  also the core assumption in the online learning class of problems known as ``bandit'' (even though our setup is a bit different, and we obtain faster rates than in bandit optimization).

Again, strong convexity has emerged as a key property \citep{Hazan2014StrongSmooth}. Following \citet{polyak1990optimal}, \citet{dippon2003} or \citet{SahaTewari11} (for the traditional concept of smoothness) we show that in the large variety of online settings, high-order smoothness, namely the boundedness of high-order derivatives, may be used, with the extreme case of infinitely differentiable functions, for which the rates attain the ones for first-order oracles.
 
More precisely, throughout this paper, we consider a sequence of  convex functions $f_n: \rb^d \to \rb$, $n \geq 1$ and a convex constraint set $K \subset \rb^d$ with non-empty interior. The objectives are to output a sequence of a sequence of points $\{x_n\}_{n=0,\ldots,N} \in K$  and of queries $\{y_n\}_{n=1,\ldots,N} \in \rb^d$ to a noisy zero-th order oracle,  in order to minimize one of the following criteria:

   \BIT
  \item \emph{Stochastic optimization}: All functions $f_n$  are equal to $f$, and the goal is to minimize $$\displaystyle f(x_N) - \inf_{x \in K} f(x)$$ for the final point  $x_N \in K $. 

\item \emph{Online optimization}: The criterion to optimize, usually referred to as the ``regret'', is  $$\displaystyle \frac{1}{N} \sum_{n=1}^{N}  f_{n}( x_{n-1} ) - \inf_{ x \in K  }
  \frac{1}{N} \sum_{n=1}^N  f_n(x).$$ We immediately emphasize here that a bound valid for online optimization immediately transfers into into a bound for stochastic optimization with the choice $x_N=\frac{1}{N}\sum_{m=0}^{N-1}x_m$.
 \item \emph{Bandit learning}: this setting is similar to the online optimization case, except that the evaluation point must be equal to the query point, i.e., $y_{n+1}=x_n$ for all $n$.
  \EIT
Formally, the timing of the optimization scheme is the following. The algorithm first outputs $x_0 \in K$ and queries $y_1 \in \rb^d$. After getting $f_1(y_1) + \varepsilon_1 \in \rb$ as a feedback (where $\varepsilon_1 \in \rb$ is some noise), it outputs $x_1 \in K$ and queries $y_2 \in \rb^d$, gets $f_2(y_2) +\varepsilon_2 \in \rb$ as feedbacks, etc. Formally, let $\mathcal{F}_{n-1}$ be the $\sigma$-field generated by $\{x_0,x_1,y_1,\varepsilon_1,\ldots,x_{n-1}, y_{n-1}, \varepsilon_{n-1}\}$.
Then $x_n$ and $y_n$ are random variables adapted to $\mathcal{F}_{n-1}$ and $\varepsilon_n$ is adapted to $\mathcal{F}_n$.

For simplicity we assume that the noise is independent in the sense that the distributions of $\varepsilon_n$ conditionally to $\mathcal{H}_n$ are independent but we do not assume that the noise is identically distributed (as the distribution may depend on $y_{n-1}$, which is key for online supervised learning). Moreover,  we  assume that the noise has bounded variance $\sigma^2$ that is not necessarily known in advance (improved bounds  would be obtained if we allow dependency of algorithms in that term).  Note that martingale assumptions common in stochastic approximation~\citep{kuku} could be used  instead of conditional independence.

Motivating examples for the optimization case are  (a) simple additive noise on $f$, or (b) $f_n(x) = \E_a g(a,x)  $ and $\varepsilon_n = g(a_n,x) - \E_a g(a,x)$ for $a_n$ a random variable, which corresponds to online supervised learning where $a_n$ represents the data received at time $n$. 
 
We shall also consider the case where we essentially query twice the same functions before outputting a new point $x_{n+1}$; we stress out here that the two feedbacks are two noisy evaluations where \emph{the noises are independent}, as opposed to \citet{agarwal2010optimal,duchi}. As a consequence, the classical optimization setup remains identical except that we make $2N$ queries instead of $N$, thus rates of convergence are independent of this trick. As a consequence, it only makes a difference in the online optimization setup, where we now need to assume that the same function is observed twice in a row.

We introduce this  two-point setting as it allows us to consider the case where the constraint set is the whole space~$\rb^d$. Moreover, the algorithms  do not need to perform a projection at each step and rates of convergence are independent of the maximal value of the loss functions (which should not appear as the problem in translation invariant). Note that (a) this unconstrained setting is common in smooth optimization, and (b) that our proof technique can extend to composite optimization where a non-smooth term is added with its proximal operator~\citep{xiao2010dual,hu2009accelerated}. On the other hand, when the constraint set is a compact convex subset, of diameter denoted by $R>0$, then we shall use a classical ``one-point'' algorithm that queries each $f_n$ only once.

\medskip

We shall provide algorithms and explicit rates of convergence for all the following cases 
\BIT
\item[i)] Unconstrained ($K = \rb^d$) vs. constrained optimization ($K$ is compact convex).
\item[ii)] Convex vs. $\mu$-strongly convex mappings.
 \item[iii)] Stochastic optimization vs. online optimization. 
\EIT
Maybe surprisingly, as shown in Figure \ref{FIG:rates}, rates of convergence  are actually independent of the unconstrained/constrained setting and on the stochastic vs. online case, at least when $f_n$ are Lipschitz-continuous which is a required setup for online optimization. 
 We emphasize here that the asymptotic dependencies in $N$ and $d$ are exact, i.e., no logarithmic terms are hidden. 

Note that we do not consider here the bandit setting that imposes that $x_{n+1}=y_n$. This can be deduced from Figure \ref{FIG:rates} as the rate for strongly convex functions would violate the lower bound of \citet{shamir2013complexity} for bandit learning.

\begin{figure}
\begin{center}
\begin{tabular}{r|c|c|}
\multicolumn{1}{c}{}& \multicolumn{1}{c}{Stochastic} &  \multicolumn{1}{c}{Online} \\
\multicolumn{1}{c}{\phantom{\Large I}}& \multicolumn{1}{c}{Constrained \& Unconstrained}  &  \multicolumn{1}{c}{Constrained \& Unconstrained\vspace{0.3cm}}\\
\phantom{\huge I}Convex $\beta=2$   &  $(\frac{d^2}{N})^{\frac{1}{3}}$   & $(\frac{d^2}{N})^{\frac{1}{3}}$  \\
\phantom{\huge I}\phantom{Convex} $\beta>2$   &  $(\frac{d^2}{N})^{\frac{\beta-1}{2\beta}} $  &$(\frac{d^2}{N})^{\frac{\beta-1}{2\beta}} $ \vspace{0.2cm}\\
\phantom{\huge I}$\mu$-stg convex $\beta=2$ &  $\sqrt{\frac{d^2}{\mu N}}$ &$\sqrt{\frac{d^2}{\mu N}}$\\
\phantom{\huge I}\phantom{$\mu$-stg Convex} $\beta>2$ &  $(\frac{d^2}{\mu N})^{\frac{\beta-1}{\beta+1}}$   &$(\frac{d^2}{\mu N})^{\frac{\beta-1}{\beta+1}}$ \\
\phantom{\huge I}\phantom{$\mu$-stg Convex} (asymptotic)& $\frac{1}{\mu^2}(\frac{d^2}{ N})^{\frac{\beta}{\beta+1}}$&
\end{tabular}
\end{center}
\label{FIG:rates}
\caption{Summary of the principal rates of convergence achieved by our algorithms for stochastic or online optimization. The bounds in the last asymptotic regime are only true when $N$ is large enough and are only valid for stochastic optimization.}\end{figure}

\subsection{Smoothness assumption}

We shall assume that all mappings in question are defined on $\rb^d$ and almost surely $(\beta\!-\!1)$-times differentiable and that for all $\| v \|_2 = 1$, and $x,y \in \rb^d$, then 
\BEQ
\| f^{(\beta-1)}(x) v^{\beta-1} -  f^{(\beta-1)}(y) v^{\beta-1} \|_2 \leq M^\beta_\beta \| x - y\|_2,
\label{eq:smoothdef1}
\EEQ where we define
$$
f^{(m)}(x) v^m 
= \sum_{m_1 + \cdots + m_d = m } \frac{ \partial^m f}{\partial^{m_1} x_1 \cdots \partial^{m_d} x_d } v_1^{m_1} \cdots v_d^{m_d}
$$
as the $m$-th term in the Taylor expansion of $f$. We refer to such functions as \emph{$\beta$-th order smooth} functions. Note that a stronger assumption is that $f$ is $\beta$-times differentiable with a uniform bound 
\BEQ
\sup_{x \in \rb^d} \sup_{ \| v\| \leq 1} | f^{(\beta)}(x) v^\beta| \leqslant M_\beta^\beta.
\label{eq:smoothdef2}
\EEQ
These notions extends the traditional smoothness, which corresponds to $\beta=2$~\citep{nest2004}. Notice that this implies that for all $x,y$ (as a consequence of Taylor expansions with integral remainder):
\begin{equation}\label{EQ:DefMbeta}
\bigg| f(y) - \sum_{|m| \leq \beta - 1} \frac{1}{m!} f^{(m)}(x)(y-x)^m
\bigg| \leq \frac{M^\beta_\beta}{\beta!}  \| y - x \|^\beta.
\end{equation}
We emphasize the fact that high-order smoothness, in the sense defined above, implies lower order smoothness only if mappings are defined on a compact set. If a mapping is defined on the whole space, then it can be second order smooth without being first order smooth, such as any non trivial quadratic function.
\medskip

We now mention the following lemma that relates the different degrees of smoothness of $f$.
\begin{lemma}\label{LEM:Poincare}
Let $f: K \to \rb$ be a continuous mapping  that is $\beta_1$-smooth and $\beta_2$-smooth, with the associate constants $M_{\beta_1}$ and $M_{\beta_2}$, where $\beta_1 < \beta_2$. Then $f$ is $\beta$-smooth for all $\beta \in [\beta_1,\beta_2]$ and there exist a sequence of weights $\alpha_\beta$, for all $\beta \in [\beta_1,\beta_2]$, independent of $M_{\beta_1}$ and $M_{\beta_2}$ such that  
$$\alpha_\beta M_\beta^\beta \leq 2 (\alpha_{\beta_1} M_{\beta_1}^{\beta_1})^{\frac{\beta_2-\beta}{\beta_2-\beta_1}}(\alpha_{\beta_2} M_{\beta_2}^{\beta_2})^{\frac{\beta-\beta1}{\beta_2-\beta_1}}$$

In particular,
\begin{itemize}
\item[i)] if $K$ is compact then $f$ is bounded (i.e., $0$-smooth). As a consequence, $\beta$-smoothness immediately entails that $f$ is Lipschitz and $2$-smooth.
\item[ii)] If $f$ is Lipschitz and $\beta$-smooth (for $\beta  \geq 2$), then $f$ is $2$-smooth.
\end{itemize}
\end{lemma}
From now on, we shall assume that all mappings $f_n$ are $\beta$-smooth, for some $\beta \geq 2$, with a common associated constant $M_\beta$ which is known (which typically holds in many settings, see   next example).  In online unconstrained optimization, we will also impose that $f_n$ is Lipschitz (again, this is automatic when $K$ is compact).

\paragraph{Special case: logistic regression.}
If $f(x) = \E_a \log( 1 + \exp(- a^\top x) )$ for a certain random vector~$a \in \rb^d$ which is uniformly bounded by $R$, then we consider $\varepsilon_n = \log( 1 + \exp(- a_n ^\top x) ) - \E_a \log( 1 + \exp(- a^\top x) )$ for a sample $a_n$. This is online logistic regression, for which 
 the constant $M_\beta^\beta$ may be chosen to be equal to $\frac{1}{4} (\beta-1)! R^\beta$~\citep{tewari}, which is such that $M_\beta \leq \beta R$. Note that such a setting should extend to all generalized linear models. Moreover, we use a  property of logistic regression which is different than self-concordance~\citep{bach2010self}, which bounds the third derivatives  by the second derivative; it would be interesting to see if the two analyses can be combined.
 
\subsection{Related work}

As already mentioned, there is a huge (and actually still increasing) literature on stochastic optimization with zero-th order feeback  and/or on  convex bandits problem. 
We also investigate here the online optimization setup, an ``intermediate framework'' where the sequence of mappings $f_n$ can evolve adversarially but, as in optimization, the loss might be evaluated at another point than the query sent to the oracle.

We emphasize these differences between set-ups as the complexity of  stochastic zero-th order optimization and the convex bandit problem have been widely studied recently \citep{recht2012query,shamir2013complexity}. It has been observed that minimax rates of convergence in bandit problems and stochastic optimization might differ, which is not the case in our setting for our upper-bounds (one can therefore conclude that the complexity of convex bandits is not hidden in the evolving sequence of loss functions, but more importantly on the constraint that the query point is where the loss is evaluated).

Moreover, it has also been shown by \citet{recht2012query,shamir2013complexity} that the slow rates of $\sqrt{d^2/n}$ are minimax optimal for stochastic optimization or convex bandits. The optimal rates of $\sqrt{1/n}$ have been obtained \citep{nemirovsky1983problem,LiangNarayananRakhlin14}  but without the explicit dependency in the dimension $d$; moreover, those techniques cannot be used in online optimization. The lower bound in $\sqrt{d^2/n}$ holds even if the mappings are highly regular, as quadratic and strongly-convex \citep{shamir2013complexity}.  However, in that case, the optimization error decreases as $d^2/n$; see also \citet{HazanKorenLevy14}  for a similar result on logistic regression. This result\footnote{Actually, the quadratic case is very particular as we could show   that one can query points arbitrarily  away from the origin to reduce variance.} can be interpreted as  an extreme case of our regularity assumptions, i.e., when $\beta=+\infty$ or $M_3=0$. As a consequence, we somehow interpolate between the well studied extreme problems in online learning with either smooth  or  quadratic mappings. 

The intermediate framework between smooth and quadratic (or mappings infinitely differentiable) has also been studied by \citet{fabian1967}, \citet{Chen_1988} and \citet{polyak1990optimal} where the focus was stochastic optimization with the objective of bounding the error   in the argument and not  in function evaluation. \citet{fabian1967} obtained an algorithm such that the distance to the maximum is of the order of $N^{-\frac{\beta-1}{2\beta}}$ which is optimal \citep{Chen_1988}. In the case of strongly-convex mappings, this has been improved by \citet{polyak1990optimal}  to $N^{-\frac{\beta-1}{\beta}}$ which is also optimal. Our set-up is more general (as we consider also online learning, function evaluations) and we recover the aforementioned results as a byproduct of ours, with a novel non-asymptotic analysis with an explicit dependencies in the dimension and parameters of smoothness and strong convexity.

\section{Smoothing Lemma}\label{SEC:SmoothingLemma}
Our analysis relies on a novel single stochastic approximation lemma, which combines ideas from \citet{nemirovsky1983problem,nesterov2011random} and \citet{polyak1990optimal}. Let $f$ be a convex function defined on $\rb^d$. 

\paragraph{Expectation of random function evaluations around a point.} 

Given positive scalars $\delta, r>0$, we  consider sampling the value $f( x + r\delta u)$ around $x$, for $u$ uniformly distributed in the unit \emph{sphere} for the Euclidean norm. As shown by~\citet{nemirovsky1983problem}, the expectation of the vector $f( x + r\delta u) u $ is equal to $ {d}/(\delta r)$ times the gradient of a function which is an approximation of $f$, that is, 
$x \mapsto \E_{\|v\|_2 \leqslant 1} f(x + \delta r v)$, where   $v$ is now sampled uniformly from the unit \emph{ball}. This simple result is a consequence of Stokes' theorem\footnote{Without loss of generality, we may consider $r\delta = 1$ and $\mathbb{B}$ the unit ball; then the gradient of $x \mapsto \E_{\|v\|_2 \leqslant 1} f(x +   v)$ is $\frac{1}{{\rm vol}(\mathbb{B})} \int_\mathbb{B} f'(x+v) dv  = \frac{1}{{\rm vol}(\mathbb{B})} \int_{\partial \mathbb{B} } f'(x+u) du$ by Stokes' theorem and because $u$ a normal vector to the unit sphere $\partial \mathbb{B}$ at $u$. The factor of $d$ comes from the ratio between the volume of the ball and the surface of the sphere. }
Thus the expectation of function evaluations at random points around $x$ is the gradient of a certain function. This is a key property which is used by most non-asymptotic analyses~\citep{flaxman2005online} of zero-th order optimization.

\paragraph{High-order smoothness and gradient evaluation.}
As shown by~\citet{polyak1990optimal} in one dimension (and then generalized to partial derivatives), if we now sample independently $r$ from the uniform distribution in $[-1,1]$, and we consider a function $k(r)$ such that $\E_r r k(r)=1$ and $\E_r r^k k(r)=0$ for $k$ odd between $3$ and $\beta$, then $\frac{1}{\delta} f(x+\delta r ) k(r)$ is a good approximation of the derivative of $f$ at $x$, with an expectation (with respect to $r$) which is equal to $f'(x)$ up to terms of order $\delta^{\beta-1}$ if $f$ is $\beta$-th order smooth.

In the following lemma, we combine these two ideas (see proof in Appendix~\ref{app:smoothing}):
\begin{lemma}
\label{lemma}
Let $f:\rb^d \to \rb$ a convex function.
 Define 
 $$\hat{f}_\delta(x)
= \E_r \E_{\| v\| \leq 1} f(x + r \delta v) r k(r), $$ where the expectation is taken with respect to the uniform distribution on the unit ball for $v$, and $r \in \rb$ is independent from $v$, with uniform distribution in $[-1,1]$, and $k(r)$ is such that $\E_r r k(r)=1$ and $\E_r r^k k(r)=0$ for  $k$ odd between $3$ and $\beta$. Then, $\hat{f}_\delta$ is differentiable and for any $x \in \rb^d$,
 \BEQ
\label{eq:lemma}
\hat{f}'_\delta(x) = \frac{d}{\delta} \, \E_r \E_{\| u\|_2 = 1} \big[  f(x+\delta r u ) k(r)  u \big].
\EEQ
Moreover, we  have the approximation bounds (the second being valid if $f$ is differentiable):
\BEAS
\big| \hat{f}_\delta(x) - f(x) \big| 
& \leq &  \frac{M^\beta_\beta}{\beta!}\delta^\beta \Big( \E_r |k(r) r^{\beta+1}| \Big),\\
\| \hat{f}_\delta' - f'(x) \|
& \leq &    \frac{M^{\beta}_{\beta}}{(\beta-1)!}\delta^{\beta-1} \Big( \E_r |k(r) r^{\beta}| \Big) .
\EEAS
\end{lemma}

\paragraph{Choice of $k(r)$.}
Following \citet{polyak1990optimal}, we consider $r$ uniformly distributed in $[-1,1]$. For $\beta \in \{1,2\}$,
we may take $k(r) = 3 r $, for which we have
$\E r k(r) =    \frac{1}{2} \int_{-1}^{1} 3 r^2 dr =  1 $.

Consider  orthonormal polynomials $p_m(\cdot)$ for the distribution on $r$, i.e., such that $\E_r p_mp_{m'}=0$ for $m\neq m'$, $\E_r p_m^2=1$ and $p_0(\cdot), \ldots, p_s(\cdot)$ spans the vector space of polynomials of degree less or equal than $s$, for all $s \in \N$.

Then we may choose
$k(r) = \sum_{m=0}^{\beta} p_m'(0) p_m(r)$. Indeed, following \citet{polyak1990optimal}, given $s \in \N$, let $b_0,\ldots,b_s$ be the coordinates of $r^s$ in the chosen basis, i.e., $r^s = \sum_{j=0}^s b_j p_j(r)$, then 
$\E_r k(r) r^s = \sum_{j=0}^s b_j p_j'(0) = 0$ for $s \neq 1$ and zero for $s \in \{0,2,\dots, \beta  \}$. Note that this is more than we actually need as in Lemma~\ref{lemma}, we only need $s$ being odd. 

We have, for $r$ uniform in $[-1,1]$,
$
p_m(u) = \sqrt{ 2m+1 } L_m(u)
$ where $L_m$ is the $m$-th Legendre polynomial. For example, we have the following values  for  
  $\beta \in \{1,2,3,4,5,6\}$:
\BEAS
k_1(r) = k_2(r) & = & 3r \\
k_3(r) = k_4(r) & = & \frac{15r}{4} ( 5 - 7 r^3) \\
k_5(r) = k_6(r) & = & \frac{195r}{64} ( 99 r^4 - 126 r^2 + 35) .
\EEAS

\paragraph{Bounds.}
In this paper, we also need the following bounds, which are shown in Appendix~\ref{app:legendre} by using properties of Legendre polynomials:
\BEAS
\E_r |k(r)|^2 & \leq & 3 \beta^3 \\
\E_r |k(r)|^2r^2 & \leq & 8 \beta^2 \\
\E_r |k(r) r^{\beta+1} | & \leq & 2 \sqrt{2} \beta .
\EEAS

\paragraph{Convexity.}
With respect to the kernel chosen, $\hat{f}_\delta$ is always convex for $\beta=2$, because $rk(r)$ is always non-negative. For $\beta \geq 3$, if $f$ is $\mu$-strongly-convex, then $\hat{f}_\delta$ is $\mu/2$-strongly-convex if $\delta$ is small enough.

Indeed, by definition of $\hat{f}_\delta$ and by $3$-smoothness of $f$, we obtain that
$$D^2\hat{f}_\delta(x)
= \E_r \E_{\| v\| \leq 1} D^2f(x + r \delta v) r k(r)  \succcurlyeq \mu I_d - \delta M_3^3 \E_r |k(r)|r^2 J_d,$$ 
where $J_d$ is the matrix whose components are all equal to 1. As a consequence, $\hat{f}_\delta$ is $\mu/2$-strongly-convex as soon as $\delta \leq 16\mu/(d\beta^2 M_3^3)$.
Note however that $\hat{f}_\delta$ is not convex in general.

\section{Unconstrained  Optimization}

We recall that  $f_n=f$ in this setting and that we chose to make two queries $y_{n_-}, y_{n_+}$ of $f$ before outputting the next point $x_n$. Of course,  \textsl{stricto sensu}, one should replace $N$ by $N/2$ in our rates of convergence. For simplicity and consistency in proofs, we chose to keep the formulation as  $N$ stages of 2 queries. Moreover, the two independent noises  can be combined into a single one.

\medskip

We thus consider two-point algorithms of the form
\BEQ
\label{eq:twopoint}
x_{n} = x_{n-1} - \gamma_n  \frac{d}{2 \delta_n}   \big[  f(x_{n-1}+\delta_n r_n u_n ) 
-  f(x_{n-1}- \delta_n  r_n u_n )  + \varepsilon_n \big] k(r_n)  u_n ,
\EEQ
where $\gamma_n$  and $\delta_n$  are constants that depend on $n$, $u_n$ is uniform in the unit-sphere, and $k(r_n)$ satisfies the conditions of Lemma~\ref{lemma}. We emphasize again that the noise is different at the two evaluations points $y_{n_-}=x_{n-1}- \delta_n  r_n u_n$ and $y_{n_+}=x_{n-1}+ \delta_n  r_n u_n$ and do not cancel by differencing (the random variable $\varepsilon_n$ is thus the difference of these two zero-mean independent noises). We define $\bar{x}_{n-1} = \frac{1}{n} \sum_{k=0}^{n-1} x_k$ as the averaged iterate.

\subsection{Convex Mappings}
We first consider the case of convex (i.e., not necessarily strongly-convex) mappings. In order to preserve the flow of the paper, we delay the proof to Appendix~\ref{app:twoppoint-optim}. 
\begin{proposition}[Unconstrained, Convex]
\label{prop:twoppoint-optim}
Assume  $f$ is (a) $\beta$-th order smooth with constant $M_\beta$, and (b) 2nd-order smooth with constant $M_2$.

Consider the algorithm in \eq{twopoint}, with 
  $\gamma_n = \gamma = \frac{1}{24 d^{(\beta-1)/\beta} M_2^2 \beta^2  N^{(\beta+1)/(2\beta)}}
$
and $\delta_n  =  \delta =  \frac{\beta d^{1/\beta}}{N^{1/(2\beta)} } (M_\beta^\beta M_2)^{-1/(\beta+1)}$ for $n \in \{1,\dots,N\}$.
Then, $ \E f(\overline{x}_{N-1}) - f(x^*) 
$ is less than 
$$ \Big(\frac{d^2  }{N} \Big) ^{(\beta-1)/(2\beta)}
 \bigg( 7   \beta M_2 \| x_{0}-x_\ast \|  
 + 3 \sigma +    (M_\beta/  M_2)^{2\beta/(\beta+1)}
 +   \frac{\beta}{N^{1/\beta} } (M_\beta/M_2)^{-\beta/(\beta+1)}
 \bigg)^2.$$
\end{proposition}
We can make the following observations about this proposition:
\BIT
\item \textbf{Dominating term in $ \Big(\frac{d^2  }{N} \Big) ^{(\beta-1)/(2\beta)}$}: the second term in the bound above is asymptotically negligible when $N$ grows
 and we recover the same scaling as the one-point estimate late in Section~\ref{sec:onepoint}, with the same scalings for the step size. 
 
 \item  \textbf{Recovering the optimal rate of $\frac{1}{\sqrt{N}}$:} If $\beta$ is infinite then one can consider $\beta = \log_2(N)/2$ to recover the optimal rate  (up to logarithmic factor) since  $2\sqrt{N}^{\beta/(\beta+1)} \geq \sqrt{N}$. Formally, the rate of convergence would also depend on $M_{\log_2(N)/2}$ that has to grow slowly;  for logistic regression, this term is also logarithmic.

 This rate is also achieved if $M_\beta=0$, a situation that can occur if  $f$ is a polynomial, by taking $\delta$ of the order of a constant and $\gamma$ of the order of $1/\sqrt{N}$. 

\item \textbf{Anytime version}: as shown in Appendix~\ref{app:twoppoint-optim}, by using decaying step-sizes, we obtain an anytime result (i.e., a result valid for all $N \in \mathbb{N}$) with an extra factor of $\log(N+1)$.

\EIT

\subsection{Strongly-Convex Mappings}
We now consider the case of $\mu$-strongly-convex mappings. We emphasize here that, in the following proposition, fast rates of convergence are achieved with  non-uniform averages, i.e., we introduce $\widehat{x}_{n-1}=\frac{2}{n(n+1)}\sum_{k=0}^{n-1} (k+1) x_k$. We again delay the proof to Appendix \ref{app:twoppoint-optim-strongly}.

\begin{proposition}[Unconstrained, Strongly-convex, 2-smooth]
\label{prop:twoppoint-optim-strongly}
Assume  $f$ is (a) $\beta$-th order smooth with constant $M_\beta$, and (b) 2nd-order smooth with constant $M_2$.

Consider the algorithm in \eq{twopoint}, with $\gamma_n = \frac{1}{\mu n} $
and $\delta_n = \bigg( \frac{d^2 \beta!}{ M_\beta^\beta \mu n} \bigg)^{1/(\beta+1)}$, for $n \in \{1,\dots,N\}$.
Then, $ \E f(\widehat{x}_{N-1}) - f(x^*) 
$ is less than 
$$    \big( \frac{d^2 M_\beta^2}{n \mu } \big) ^{(\beta-1) / ( \beta+1 )}
 \bigg( 8   \beta M_\beta \| x_{0}-x_\ast \|  
 + 4 \sigma +   2  +    {\beta }  (M_2/M_\beta)^{2}  \big( \frac{M_\beta^2}{n \mu } \big) ^{2 / ( \beta+1 )}
 \bigg)^2.$$
 \end{proposition}
We emphasize here that the first bound allows to recover the previous bound for the optimization of a  non-strongly-convex mapping $f$ by using the aforementioned scheme to $f + \mu \|\cdot\|^2$ and let $\mu$ depend on $n$. The second bound has the optimal dependency in $N$ but a worse dependency in $\mu$.

\section{Constrained  Optimization}
\label{sec:onepoint}
In this setup, where the constraint set $K$ is compact convex and of diameter~$R$, we use a classical one-point algorithm:
\begin{equation}\label{eq:onepoint}
x_n = \Pi_K \Big( x_{n-1} - \gamma_n  \frac{d}{\delta_n}   \big[  f(x_{n-1}+\delta_n r_n u_n ) + \varepsilon_n \big] k(r_n)  u_n \Big),
\end{equation}
where the parameter $\gamma_n$ and $\delta_n$ can  evolve with time. In particular, we have $y_n= x_{n-1}+\delta_n r_n u_n$.

\subsection{Convex Mappings}
Again, we begin with the case of convex (i.e., non necessarily strongly-convex) mappings. The proof of the following proposition is delayed to Appendix  \ref{app:onepoint-nonstrong}.

\begin{proposition}[Constrained, Convex]
\label{prop:onepoint-nonstrong}
Assume  $f$ is $\beta$-th order smooth with constant $M_\beta$ and consider the algorithm in \eq{onepoint}, with $\gamma_n = \frac{R \delta_n }{\sqrt{ \beta^3}d \sqrt{n}}$ and $ \delta_n^{\beta } = 
\frac{    d   \sqrt{\beta}(\beta-1) ! }{
\sqrt{n} M^\beta_\beta }$, for $n \in \{1,\dots,N\}$.
Then, $ \E f(\overline{x}_N) - f(x^*) $ is less than 
$$
25R M_\beta \bigg(\frac{ d^2 \beta}{ N}\bigg)^{\frac{\beta-1}{2\beta}}(C_{\delta_1}+\sigma^2+1),$$
where $C_\delta$ is a uniform bound of $f$ on the $\delta$-neighborhood of $K$.\end{proposition}

We can make the following observations:
\BIT
\item \textbf{Anytime algorithm:} The algorithm is independent of $N$, thus it is anytime, i.e., the above rate holds for all $N \in \mathbb{N}$. Notice also that $C_{\delta_1}$ can actually be replaced, asymptotically, by~$C_0$; see the proof in Appendix \ref{app:onepoint-nonstrong}.
\item  \textbf{Upper-bounding $C_\delta$:} Since the mapping $f$ is bounded on the compact set $K$ and  $\beta$-smooth, it is necessarily $M_1$-Lipchitz. Then  $C_\delta$ is bounded by $C_0+M_1\delta$;

\item \textbf{Concerning the unknown quantities ($C_\delta$ and $\sigma^2$):} The step-sizes do not depend on the unknown quantities $C_\delta$ or $\sigma^2$. However, if they are known, then the dependency on $C_0$ and $\sigma^2$ can be slightly improved. Similarly, we assumed that the constant $M_\beta$ was known. If it is not the case, the algorithm still works with the specific choice of $\delta_n^\beta=dR\sqrt{\beta}(\beta-1)!/\sqrt{n}$; the dependency in $M_\beta$ would be changed from $M_\beta$ into~$M_\beta^\beta$.
\EIT

 \subsection{Strongly-Convex Mappings}

Similarly to the unconstrained case, we now consider the case of $\mu$-strongly-convex mappings where rates can be improved. As before, we delay the proof of the following proposition to Appendix \ref{app:oneppoint-onlinestrongly}.

\begin{proposition}[Constrained, Strongly-convex]
\label{prop:oneppoint-online-strongly}
Assume  $f$ is $\beta$-th order smooth with constant $M_\beta$. Consider the algorithm in \eq{onepoint}, with $\gamma_n = 1/(n\mu) $
and $\delta_n=\left(\frac{  d^2\beta\beta!}{n\mu M^\beta_\beta  }\right)^{\frac{1}{\beta+1}}$   for $n \in \{1,\dots,N\}$.
Then, $ \E f(\overline{x}_N) - f(x^*) $ is less than 
$$
 15\beta^2M_\beta^{\frac{2\beta}{\beta+1}}\bigg(\frac{d^2}{\mu N}\bigg)^{\frac{\beta-1}{\beta+1}}(C_{\delta_1}+\sigma^2+1)\, .
 $$

\end{proposition}

We emphasize the fact that the algorithm is again independent of $N$, thus the result  is actually anytime.

\section{Refined Upper and Lower Bounds}
In this section, we consider improved bounds in the smooth case ($\beta=2$),  as well as asymptotic and lower bounds for  strongly-convex mappings for all $\beta$.

As mentioned at the end of Section \ref{SEC:SmoothingLemma}, if $\beta= 2$ then $\hat{f}_\delta$ is always convex. As a consequence, the analysis of the algorithms can be improved by noting that \eq{twopoint} and \eq{onepoint} correspond to an exact stochastic gradient descent of the approximate mapping $\hat{f}_\delta$. We recall that the analysis for $\beta\geq 3$ was based on the fact that  \eq{twopoint} and \eq{onepoint} correspond to an approximate stochastic gradient descent of $f$.

The differences between $f'$ and $\hat{f}'_\delta$ is of the order of $\delta^{\beta-1}$ while $\hat{f}_\delta$ is $\delta^\beta$-close to $f$ (disregarding the other dependencies in the dimension $d$ and smoothing parameter $\beta$). As a consequence, when $\beta=2$, we can replace the error term in $\delta^{\beta-1}$ when approximating gradients by $\delta^\beta$, as we approximate the value functions. Using this idea, and following the same lines of proof, we obtain the following  proposition (see proof in Appendix~\ref{app:8}).

\begin{proposition}[The case $\beta=2$]\label{PR:beta2}
Assume that $f$ is 2-smooth, then the algorithms described in \eq{onepoint} and \eq{twopoint}, with  adapted choices of parameters, ensures the following upper-bound on  $\E f(x_N) - f(x^\star)$:
\begin{description}
\item[--] \textsl{for unconstrained  optimization of  convex mappings}, 
$$
2
  \bigg( \frac{  d^2   }{N  } \bigg)^{\frac{1}{3}}
 \!\!\bigg( 96 M_2^2 \| x_0 -  x_\ast \|^2
+ \frac{\sigma^2}{10} + 18 \bigg)
+ \frac{  2d^2 }{ N }  
  ,
$$
\item[--]  \textsl{for unconstrained  optimization of strongly-convex mappings},
$$ 4 
    \big( 2\sigma^2 + 27 \big) 	
 \sqrt{\frac{   d^2M_2^2  \log(N)}{N \mu} }+       
\bigg(   \frac{  21 d^2 M_2^2\log(N)}{N \mu} \bigg)^{3/2 } ,
 $$
\item[--] \textsl{for constrained  optimization of  convex mappings}, 
$$
44\bigg(\frac{d^2 M_2^2R^2}{N}\bigg)^{\frac{1}{3}}(C_{\delta_1}+\sigma^2+1),
$$
\item[--] \textsl{for constrained  optimization of strongly-convex mappings}, 
$$
66 \sqrt{\frac{d^2M_2^2}{\mu N} }(C_\delta^2 + \sigma^2+1).
 $$
\end{description}
\end{proposition}
We mention here that if we had just plugged the value $\beta=2$ in the general  propositions, we would have got rates of convergence of the order of $n^{-1/4}$ and $(\mu n)^{-1/3}$, instead of $n^{-1/3}$ and $(\mu n)^{-1/2}$, respectively in the non-strongly and $\mu$-strongly-convex case.

Similarly, we have proved that if $f$ is $\mu$-strongly-convex and $\delta$ is small enough, then $\hat{f}_\delta$ is $\mu/2$-strongly-convex. As a consequence, the previous arguments hold and we can, asymptotically, obtain better rates of convergences, as we now show (see proof in Appendix~\ref{app:oneppoint-optimstrongly}).
\begin{proposition}[Asymptotics with strongly-convex mappings]\label{prop:oneppoint-optim-strongly}

Assume that $f$ is $\beta$-smooth, $\mu$-strongly-convex and globally optimized at $x^\star$ on $K$.  Then the algorithms described in \eq{onepoint} and \eq{twopoint}, with  adapted choices of parameters, ensure the following upper-bound on $\|x_N-x^\star\|$ as soon as $N$ is big enough:
\begin{description}
\item[--] \textsl{for unconstrained  optimization of strongly-convex mappings},  
$$ \frac{16   M_\beta^2 }{\mu^2}(2\sigma^2 + 16) 
\bigg( \frac{d^2 \log (N+1)}{N} \bigg)^{\frac{\beta-1}{\beta}}+
\frac{48 \beta^3 M_2^4}{\mu^2 M_\beta^2} \bigg( \frac{d^2 \log (N+1)}{N} \bigg)^{\frac{\beta+1}{\beta}},
 $$
\item[--] \textsl{for constrained  optimization of strongly-convex mappings}, 
$$
16\beta \left(\frac{d^2}{N}\right)^{\frac{\beta-1}{\beta}} \left(\frac{2eM_\beta M_2}{\mu}\right)^2\left(3C_\delta^2+3\sigma^2+1\right).
 $$
\end{description}
We recall that from those upper-bounds , we obtain  $\E f(x_N) - f(x_\ast) \leq \frac{M_2^2}{2} \E \| x_N - x_\ast\|_2^2$.
\end{proposition}
The proof is delayed to Appendix \ref{app:oneppoint-optimstrongly}. 

\bigskip

We conclude this section with a lower bound  for the optimization of strongly-convex mappings, brought to our attention by O.~Shamir and based on techniques from \cite{shamir2013complexity}. This lower bounds matches the lower bound of \cite{polyak1990optimal}, but it is non-asymptotic, quite simple and one can obtain explicit dependencies in the different parameters.  We only sketch it in one dimension, as it contains all the relevant ideas; details can be found in \cite{shamir2013complexity}.

Consider the two mappings
$$
f_1(x) = 2\mu x^2 + \alpha g\Big(\frac{x}{\theta}\Big)  \ \text{ and } f_2(x) = x^2 - \alpha g\Big(\frac{x}{\theta}\Big) , \ \text{ where } g(y) = \frac{y}{1+y^2},  
$$
and notice that $f_1(x)=f_2(-x)$, $|g(y)| \leq 1/2$ and $\big|g^{(\beta)}(y)\big|\leq 2^{\beta+1}\beta!\leq (2\beta)^\beta$. As a consequence, it is not difficult to see that $\| f_1 -f_2\|_\infty \leq \alpha$, that $f_1$ and $f_2$ are  $\beta$-th order smooth with the constant $M_\beta \leq \frac{2\alpha^{\frac{1}{\beta}}\beta}{\theta}$  and $(4\mu-\frac{3}{2}\frac{\alpha}{\theta^2})$-strongly convex,  and that $f_i(0)- f_i^* \geq  \frac{\alpha}{16\mu\theta^2}$ as soon as  $ \frac{\alpha}{\theta^2} \leq 2\mu$.

Given fixed values for the parameters $\beta$ and $M$, the choices of $\alpha = T^{-1/2}$ and $\theta = cT^{-1/2\beta}$  where $c=\frac{2\beta}{M}$ ensure that $\alpha/\theta^2 \leq 2\mu$ as soon as $T \geq (2\mu c^2)^{-\frac{2\beta}{\beta-2}}$ and that the mappings $f_1$ and $f_2$ are $\mu$-strongly convex and  $\beta$-th order smooth with a constant $M_\beta\leq M$.

Moreover, since $\| f_1 -f_2\|_\infty \leq 1/\sqrt{T}$,    $f_1$ and $f_2$ are undistinguishable with only $T$ queries 
and thus  any algorithm must suffer, when facing $f_1$ or $f_2$ an error of the order of 
$$
\min_x \max \big\{ f_1(x)-f_1^*, f_2(x) -f_2^*\big\} =f_1(0)- f_1^* \geq \frac{1}{64}\frac{M}{\mu \beta^2} T^{-\frac{\beta-1}{\beta}}.
$$

\section{Online Optimization}

In the online optimization setting, we have to modify algorithms that use non-uniform averaging as the regret is computed with respect to the Cesaro average of the losses. The online version of the algorithms are described in \eq{twopointOnline} and \eq{onepointOnline}. The difference with the algorithms of the stochastic case is simply that $f$ is replaced by $f_n$.

For the two-point algorithm, we recall that it requires that  each loss functions can be queried twice, but we emphasize again that the noise is different for the two evaluations and do not cancel simply by differencing.
\BEQ
\label{eq:twopointOnline}
x_{n} = x_{n-1} - \gamma_n  \frac{d}{2 \delta_n}   \big[  f_n(x_{n-1}+\delta_n r_n u_n ) 
-  f_n(x_{n-1}- \delta_n  r_n u_n )  + \varepsilon_n \big] k(r_n)  u_n ,
\EEQ
where $\gamma_n$  and $\delta_n$  depend on $n$. 

The 1-point algorithm  evaluates once each loss function and rewrites as
\begin{equation}\label{eq:onepointOnline}
x_n = \Pi_K \Big( x_{n-1} - \gamma_n  \frac{d}{\delta_n}   \big[  f_n(x_{n-1}+\delta_n r_n u_n ) + \varepsilon_n \big] k(r_n)  u_n \Big),
\end{equation}
where the parameters $\gamma_n$ and $\delta_n$ can  evolve with time.

\begin{proposition}\label{prop:onlinefinal}
Assume each $f_n$ is $\beta$-order smooth and $M_1$-Lipschitz. Then the online version of the algorithms described in \eq{onepointOnline} and \eq{twopointOnline}, with  adapted choices of parameters, ensures the following upper-bound on the regret $\frac{1}{N}\sum \E \big[  f_n(x_{n-1}) - f_n(x) \big]$:
\begin{description}
\item[--] \textsl{for unconstrained online optimization of  convex mapppings}, 
$$ \Big(\frac{d^2  }{N} \Big) ^{\frac{\beta-1}{2\beta}}
 \bigg( 7   \beta M_2 \| x_{0}-x_\ast \|  
 + 3 \sigma +    \Big(\frac{M_\beta}{M_2}\Big)^{\frac{2\beta}{\beta+1}}
 +   \frac{\beta}{N^{1/\beta} } \Big(\frac{M_\beta}{M_2}\Big)^{\frac{-\beta}{\beta+1}}
 \bigg)^2.$$\item[--]  \textsl{for unconstrained online optimization of strongly convex mapppings}, 
$$ 2\beta^2   \bigg( \frac{  d^2M_\beta^2}{N \mu} \bigg)^{\frac{\beta}{\beta+2}}\big(\sigma^2+6\big)
+   4 \beta^2   \frac{d^2 M_1^2\log(N+1)}{N \mu} ,
 $$
\item[--] \textsl{for constrained online optimization of  convex mapppings}, 
$$
25R M_\beta \bigg(\frac{ d^2 \beta}{ N}\bigg)^{\frac{\beta-1}{2\beta}}(C_{\delta_1}+\sigma^2+1),$$
\item[--] \textsl{for constrained online optimization of strongly convex mapppings}, 
$$    \big( \frac{d^2 M_\beta^2}{n \mu } \big) ^{\frac{\beta-1}{ \beta+1}}
 \bigg( 8   \beta M_\beta \| x_{0}-x_\ast \|  
 + 4 \sigma +   2  +    {\beta }  \Big(\frac{M_\beta}{M_2}\Big)^{2}  \Big( \frac{M_\beta^2}{n \mu } \Big) ^{\frac{2}{\beta+1}}
 \bigg)^2.$$
\end{description}
\end{proposition}

Actually, the proof are  identical in the online optimization setting than in stochastic optimization. The main difference is that we do not use the convexity of $f$ to lower-bound $\frac{1}{N}\sum \E \big[  f(x_{n-1}) - f(x) \big]$ by $\E \big[f(\overline{x}_N)-f(x)\big]$.

\section{Conclusion}
In this paper, we have considered zero-th order online optimization with a special focus on highly-smooth functions such as for online logistic regression. We considered one-point estimates and two-point estimates of the gradient (with then two independent noises). For infinitely differentiable functions, our main result leads to the same dependence on sample size as gradient-based algorithms, with an extra dimension-dependent factor. 

The present analysis could be extended in a number of ways: (a) we do not cover the bandit setting. A simple extension of our results allows us to recover existing bounds for $\beta=1$~\citep{shamir2013complexity} but we are currently unable to obtain high-smoothness improvements for $\beta>1$; (b) while the two-point analysis considers unconstrained problems, the one-point analysis still requires a compact set of constraints and queries slightly outside (in a $\delta$ band around it), which might be avoided by using barrier tools like done by~\citet{Hazan2014StrongSmooth}. Finally, (c) in the strongly-convex case, the dependence on sample size is optimal in the optimization setting~\citep{polyak1990optimal}, however, the optimality of the scaling in dimension, of the plain convex case, and beyond the optimization setting remains open.

\acks{Part of this work was performed when Francis Bach was holding the Schlumberger chair at IHES and when Vianney Perchet was a researcher at INRIA. Vianney Perchet also acknowledges fundings from the ANR under grant number ANR-13-JS01-0004 and the CNRS under grant project Parasol.}

\bibliography{zero}

\newpage
\appendix

\section{Proof of technical lemmas}
\subsection{Proof of Lemma \ref{LEM:Poincare}}
This result is rather classical and we first recall the proof when $f$ is twice continuously differentiable. By Taylor expansion, for any $x,y  \in \rb^d$ and $\lambda >0$, there exists $\zeta_+ \in [x,x+\lambda y]$ and $\zeta_- \in [x,x-\lambda y]$ such that
\BEAS f(x+\lambda y) &=& f(x) + \lambda\nabla f(x)^\top y + \frac{\lambda^2}{2} y^\top D^2 f(\zeta_+)y \\
f(x-\lambda y) &=& f(x) - \lambda\nabla f(x)^\top y + \frac{\lambda^2}{2} y^\top D^2 f(\zeta_-)y, \EEAS
This implies that 
\BEAS 
\nabla f(x)^\top y &=& \frac{f(x+\lambda y)-f(x-\lambda y)}{\lambda} + \frac{\lambda}{2} y^\top D^2 f(\zeta_-)y -  \frac{\lambda^2}{2} y^\top D^2 f(\zeta_+)y \\
& \leq & \frac{M_0}{\lambda}+\lambda M_2^2 \|y\|^2 \leq 2\sqrt{M_0M_2^2}\|y\|,
\EEAS
and this yields that $M_1 \leq 2\sqrt{M_0M_2^2}$. The general proof is obtained by introducing $\beta$ different number $\lambda_1,\ldots,\lambda_\beta$, writing the $\beta$ equations
$$
\bigg| f(x+\lambda_iy) - \sum_{|m| \leq \beta - 1} \frac{\lambda_i^m}{m!} f^{(m)}(x)y^m
\bigg| \leq \frac{\lambda_i^\beta M^\beta_\beta}{\beta!}  \| y \|^\beta,
$$
and inverting the system (which is possible if $\lambda_i$ are all distinct).
\subsection{Proof of smoothing lemma}

\label{app:smoothing}
The identity in \eq{lemma} is a consequence of the result from~\citet{nemirovsky1983problem}. 
Using the smoothness assumption, we 
 have for all $x \in \rb^d$:
\BEAS
 & & \big| \hat{f}_\delta(x) - f(x) \big|  \\
& \leq & 
\bigg|
\E_r \E_{\|v\| \leq 1} rk(r) \sum_{1\leq |m| \leq \beta - 1} 
\frac{r^{|m| } \delta^{|m|} }{m!} f^{(m)}(x)v^m
\bigg| + \frac{M^\beta_\beta}{\beta!} \delta^\beta \Big( \E_r |k(r) r^{\beta+1}| \Big)  \,  \Big( \E_{\| v\| \leq 1} \| v\|^\beta\Big)
\\
& \leq & 
\bigg|
 \sum_{1\leq |m| \leq \beta - 1} 
\frac{ \big( \E_r r^{|m|+1} k(r) \big) \delta^{|m|} }{m!} f^{(m)}(x)  \E_{\|v\| \leq 1} \big( v^m \big)
\bigg|  + \frac{M^\beta_\beta}{\beta!} \delta^\beta \Big( \E_r |k(r) r^{\beta+1}| \Big)  \,  \Big( \E_{\| v\| \leq 1} \| v\|^\beta\Big).
 \EEAS
 For $|m|$ odd, then, by symmetry of the uniform distribution on the unit ball,  $\E_{\|v\| \leq 1} \big( v^m \big)=0$. Therefore, if $\E_r r^{k} k(r) = 0 $ for $k$ odd and $ 3 \leq k \leq \beta$, we have:
 \BEAS
\big| \hat{f}_\delta(x) - f(x) \big| 
& \leq &  \frac{M^\beta_\beta}{\beta!}\delta^\beta \Big( \E_r |k(r) r^{\beta+1}| \Big).  \EEAS
 
In order to prove the following result on gradients
\BEAS
\| \hat{f}_\delta' - f'(x) \|
& \leq &    \frac{M^{\beta}_{\beta}}{(\beta-1)!}\delta^{\beta-1} \Big( \E_r |k(r) r^{\beta}| \Big), \EEAS
we first assume that the $\beta+1$-th order derivative tensor is bounded, which will be sufficient by a density argument. In this case, as shown by~\citet[p.~38]{nemirovski2004interior},  for all $x$, the $\beta$-th order tensor has projections on $\beta-1$ copies of the vector $u$ and a vector $v$ which is less than $M_\beta^\beta\|u\|^{\beta-1} \|v\|$. This implies that we can apply the function value result to the function $g(x) = f'(x)^\top v$, for any $u$. This leads to the desired result.

\subsection{Bounds on function $k(r)$}
\label{app:legendre}

From the explicit parameter expansion of Legendre polynomials, we have, for any $\alpha \geq 0$,
$$
L_{2\alpha+1}'(0) = \frac{(-1)^\alpha (\alpha+1) }{2^{2\alpha}}{ 2 \alpha + 1 \choose \alpha }
=\frac{(-1)^\alpha (2 \alpha+1) }{2^{2\alpha}}{ 2 \alpha  \choose \alpha }.
$$
Moreover, we use the following bound obtained from bounds on Catalan numbers:
$
{ 2 \alpha  \choose \alpha } \leq \frac{ 4^\alpha}{\sqrt{ \pi \alpha}}
$. This leads to
$
|L_{2\alpha+1}'(0) | \leq  \frac{ 2 \alpha+1}{ \sqrt{\pi \alpha}}
$ for $\alpha>0$, while for $\alpha = 0$, $|L_{2\alpha+1}'(0) |=1$

Moreover,  for $\beta \geq 3$:
\BEAS
\E_r |k(r) |^2 
& = & \sum_{\alpha=0}^{\lfloor (\beta - 1)/2 \rfloor}
(4  \alpha + 3) |L_{2\alpha+1}'(0) |^2
\leq 3 + \sum_{\alpha=1}^{\lfloor (\beta - 1)/2 \rfloor} 
(4  \alpha + 3) \frac{ (2 \alpha + 1)^2}{ \pi \alpha} \\
& \leq &  3 + \sum_{\alpha=1}^{\lfloor (\beta - 1)/2 \rfloor} 
7  \alpha   \frac{ (3 \alpha  )^2}{ \pi \alpha} 
\leq 3 + \frac{63}{\pi} \frac{ (\beta/2)(\beta/2+1)(\beta+1) }{6} \\
& \leq & 3 +  {21} \frac{ (\beta/2)(\beta/2+\beta/3)(\beta+\beta/3) }{6} 
= 3 + \beta^3 \frac{21 \times 5 \times 4}{36 \times 12}  \leq 3 \beta^3.
\EEAS
This is trivially valid for $\beta =1 $ and $\beta =2$.

Finally, we have for $\beta \geq 3$:
\BEAS
\E_r |k(r) |^2 r^2
& = &  \sum_{\alpha, \alpha' =0}^{\lfloor (\beta - 1)/2 \rfloor}
\sqrt{ 4 \alpha +3}
\sqrt{ 4 \alpha' +3} \E_r \big[ L_{2\alpha +1}(r) L_{2\alpha' +1}(r) r^2  \big]
L_{2\alpha+1}'(0) L_{2\alpha'+1}'(0)
\\
& = &  \sum_{\alpha, \alpha' =0}^{\lfloor (\beta - 1)/2 \rfloor}
\sqrt{ 4 \alpha +3}
\sqrt{ 4 \alpha' +3} 
L_{2\alpha+1}'(0) L_{2\alpha'+1}'(0)
\\
& & \times \E_r  \bigg[ \frac{ \big[ (2 \alpha+2) L_{2 \alpha+2}(r) + (2 \alpha+1) L_{2 \alpha}(r) \big]}{4 \alpha + 3}\bigg]
\bigg[ \frac{ \big[ (2 \alpha'+2) L_{2 \alpha'+2}(r) + (2 \alpha'+1) L_{2 \alpha'}(r) \big]}{4 \alpha' + 3}\bigg]
\\
& = &  \sum_{\alpha, \alpha' =0}^{\lfloor (\beta - 1)/2 \rfloor}
\big( \sqrt{ 4 \alpha +3}
\sqrt{ 4 \alpha' +3}  \big)^{-1}
L_{2\alpha+1}'(0) L_{2\alpha'+1}'(0)
\\
& & \times   \bigg[ 
\big[ ( 2\alpha+2)^2 + (2 \alpha+1)^2 \big] \delta_{\alpha=\alpha'}
+ (2\alpha+1) 2\alpha 
 \delta_{\alpha=\alpha'+1}
+ (2\alpha'+1) 2\alpha' 
 \delta_{\alpha' =\alpha+1}
\bigg]
\\
& = &  \sum_{\alpha =0}^{\lfloor (\beta - 1)/2 \rfloor}
 (  4 \alpha +3    )^{-1}
L_{2\alpha+1}'(0)^2 \big[ ( 2\alpha+2)^2 + (2 \alpha+1)^2 \big] 
\\
& & 
+ 2 \sum_{\alpha =0}^{\lfloor (\beta - 1)/2 \rfloor - 1}
\big( \sqrt{ 4 \alpha +3}
\sqrt{ 4 \alpha  +7}  \big)^{-1}
L_{2\alpha+1}'(0) L_{2\alpha +3}'(0) 2 \alpha  ( 2 \alpha +1)
\\
& = &  \sum_{\alpha =0}^{\lfloor (\beta - 1)/2 \rfloor}
 (  4 \alpha +3    )^{-1}
L_{2\alpha+1}'(0)^2 \big[ ( 2\alpha+2)^2 + (2 \alpha+1)^2 \big] 
\\
& & 
- 2 \sum_{\alpha =0}^{\lfloor (\beta - 1)/2 \rfloor - 1}
\big( \sqrt{ 4 \alpha +3}
\sqrt{ 4 \alpha  +7}  \big)^{-1}
L_{2\alpha+1}'(0) ^2 \frac{(2\alpha+3)  (2 \alpha+2)}{4 ( \alpha+1)^2} 2 \alpha  ( 2 \alpha +1) \\
& \leq &  (  4 \alpha +3    )^{-1}
L_{2\alpha+1}'(0)^2 \big[8\alpha^2 + 12 \alpha + 5 \big]  \bigg|_{\alpha = \lfloor (\beta - 1)/2 \rfloor}
\\
& & 
+ \sum_{\alpha =0}^{\lfloor (\beta - 1)/2 \rfloor - 1}
( 4 \alpha +3)^{-1}
L_{2\alpha+1}'(0) ^2  \bigg[
8\alpha^2 + 12 \alpha + 5 - ( 2\alpha+1) 4 \alpha
\bigg] \\
& \leq &  (  4 \alpha +3    )^{-1}
\big[
\frac{(2\alpha+1)^2}{\pi \alpha} \big]
\big] \big[8\alpha^2 + 12 \alpha + 5 \big]  \bigg|_{\alpha = \lfloor (\beta - 1)/2 \rfloor}
\\
& & 
+ \frac{5}{3} L'_1(0)
+ \sum_{\alpha =1}^{\lfloor (\beta - 1)/2 \rfloor - 1}
( 4 \alpha +3)^{-1}
\big[
\frac{(2\alpha+1)^2}{\pi \alpha} \big]
\big]   \bigg[
8\alpha^2 + 12 \alpha + 5 - ( 2\alpha+1) 4 \alpha
\bigg] \\
& \leq &  ( 2 \beta    )^{-1}
\big[
\frac{\beta^2}{\pi (\beta-2)/2} \big]
\big] \big[8\beta^2 / 4 + 12 \beta/2 + 5 \big]   
\\
& & 
+ \frac{5}{3}  
+   \sum_{\alpha =1}^{\lfloor (\beta - 1)/2 \rfloor - 1}
( 4 \alpha )^{-1}
\big[
\frac{9 \alpha^2}{\pi \alpha} \big] 9 \alpha 
\\
& \leq &   
\big[
\frac{\beta}{   \pi (\beta-2)} \big]
\big] \big[2\beta^2   + 6 \beta  + 5 \big]   
+ \frac{5}{3}  
+   
\frac{81}{4 \pi} \frac{\beta}{2} ( \beta/2+1)
\leq 2 \beta^2 + 2 \beta^2 + 5 + 5/3 + \frac{81}{16 \pi} \beta^2 \frac{5}{3} \leq 8 \beta^2
\EEAS
using the three-term recursion formula for Legendre polynomials.

We will also need the following bounds:
\BEAS
\E_r |k(r) r^{\beta+1}| 
& \leq &  \sqrt{\E_r |k(r)^2 r^{2\beta+2}| }  =  2 \sqrt{2} \beta,
 \\ \E_r |k(r) |^2  r^{2 + \gamma} 
& \leq & \E_r |k(r) |^2  r^{2  } \leq 8 \beta^2 \mbox{ for any } \gamma \geq 0.
\EEAS

\section{Analysis of classic Stochastic Gradient Descents algorithms}

We recall in this section the classical proofs of stochastic gradient descents~\citep[see, e.g.][and references therein]{bubeck2015convex}.   
We first start when the mappings $f_n$ are not necessarily $\mu$-strongly convex.
\begin{proposition}[SGD non-strongly convex] \label{PR:SGD}
The   stochastic gradient descent
\begin{equation}
x_n=\Pi_K(x_n - \gamma_n g_n) 
\end{equation}
where  $g_n$ is a biased estimate of $f'_n(x_{n-1})$, i.e., such that $\E [ g_n | \mathcal{F}_{n-1}
]= f'_n(x_{n-1}) + \zeta _n$, and $\gamma_n$ is non-decreasing
achieves the following guarantee
$$
\frac{1}{N}\sum_{n=1}^N 
\E \big[  f_n(x_{n-1}) - f_n(x) \big]\leq \frac{\max_n \E\|x_n-x\|^2}{2\gamma_N} +  \frac{1}{N} \sum_{n=1}^N\E \zeta_n^\top(x_{n-1}-x)
+\frac{1}{N}\sum_{n=1}^N  \gamma_n^2 \E \|g_n\|^2\, .
$$

In particular, if $f_n=f$ and $x^\star$ is a minimizer of $f$, we obtain
$$
\E f(\overline{x}_{N-1}) - f(x^\star) \leq \frac{\max_n \E\|x_n-x\|^2}{2N\gamma_N} +  \frac{1}{N} \sum_{n=1}^N\E \zeta_n^\top(x_{n-1}-x)
+\frac{1}{2N}\sum_{n=1}^N  \gamma_n \E \|g_n\|^2
$$

\end{proposition}
\begin{proof}
We have for any $x \in K$, since projecting reduces distances,
\BEAS
\| x_n - x\|^2
& \leq & \| x_{n-1}-x\|^2 - 2 \gamma_n  g_n + \gamma_n^2\|g_n\|^2 \\
\E \| x_n - x\|^2 & \leq& \E \| x_{n-1}-x\|^2 - 2 \gamma_n  \E f_n'(x_{n-1})^\top ( x_{n-1} - x ) +  2\gamma_n \E \zeta_n^\top(x_{n-1}-x)+  \gamma_n^2 \E \|g_n\|^2 \\
& \leq & \E \| x_{n-1}-x\|^2 - 2 \gamma_n   \E \big[  f_n(x_{n-1}) - f_n(x) \big]+ 2\gamma_n \E \zeta_n^\top(x_{n-1}-x)+ \gamma_n^2 \E \|g_n\|^2\, .
  \EEAS
This  leads to
\BEAS\frac{1}{N}\sum_{n=1}^N 
\E \big[  f_n(x_{n-1}) - f_n(x) \big]
&\leq&  \frac{1}{N} \sum_{n=1}^N \frac{ \E \| x_{n-1}-x\|^2- \E \| x_{n}-x\|^2}{2\gamma_n}\\ &+ & \frac{1}{N} \sum_{n=1}^N\E \zeta_n^\top(x_{n-1}-x)
+\frac{1}{2N}\sum_{n=1}^N  \gamma_n \E \|g_n\|^2\\
& \leq& 
\frac{\max_n \E\|x_n-x\|^2}{2N\gamma_N} +  \frac{1}{N} \sum_{n=1}^N\E \zeta_n^\top(x_{n-1}-x)
+\frac{1}{2N}\sum_{n=1}^N  \gamma_n \E \|g_n\|^2\, .
\EEAS

\end{proof}

When the mappings $f_n$ are $\mu$-strongly convex, rates are improved as claimed by the following proposition.
\begin{proposition}[SGD $\mu$-strongly convex]\label{PR:SGDstrong}
The   stochastic gradient descent
\begin{equation}
x_n=\Pi_K(x_{n-1} - \gamma_n g_n) 
\end{equation}
where  $g_n$ is a biased estimate of $f'_n(x_{n-1})$, i.e., such that $\E [ g_n | \mathcal{F}_{n-1}
]= f'_n(x_{n-1}) + \zeta _n$.

\begin{itemize}
\item The choice of $\gamma_n=\frac{1}{\mu n}$ gives
\begin{equation}
\frac{1}{N}\sum_{n=1}^N\E f_n(x_{n-1})-\E f_n(x) +  \frac{\mu}{2}\|x_N-x\|^2   \leq  \frac{1}{N}\sum_{n=1}^N \E \zeta_n^\top(x_{n-1}-x) +  \frac{1}{2N}\sum_{n=1}^N \frac{\E \|g_n\|^2}{\mu n} 
\end{equation}

In particular, if $f_n=f$ and $x^\star$ is a minimizer of $f$, we obtain
$$
\E f(\bar{x}_{N-1})- f(x^\star) + \frac{\mu}{2}\|x_N -x^\star\|^2 \leq  \frac{1}{N}\sum_{n=1}^N \E \zeta_n^\top(x_{n-1}-x) +  \frac{1}{2N}\sum_{n=1}^N \frac{\E \|g_n\|^2}{\mu n}   \ .
$$
\item The choice of $\gamma_n=\frac{2}{\mu(n+1)}$ gives
\begin{equation}
\E f(\hat{x}_{N-1})- f(x^\star) + \E\| x_{n}-x\|^2\frac{\mu}{2} \leq \frac{2}{N(N+1)}\sum_{n=1}^N\E \zeta_n^\top(x_{n-1}-x) +   \frac{1}{\mu (n+1)} \E \|g_n\|^2,
\end{equation}
where $\hat{x}_{N-1}=\frac{2}{N(N+1)}\sum_{n=1}^Nn x_{n-1}$.
\end{itemize}
\end{proposition}
\begin{proof}
We have for any $x \in K$:
\BEAS
\| x_n - x\|^2
& \leq & \| x_{n-1}-x\|^2 - 2 \gamma_n  g_n + \gamma_n^2\|g_n\|^2 \\
\E \| x_n - x\|^2 & \leq& \E \| x_{n-1}-x\|^2 - 2 \gamma_n  \E f_n'(x_{n-1})^\top ( x_{n-1} - x ) +  2\gamma_n \E \zeta_n^\top(x_{n-1}-x)+  \gamma_n^2 \E \|g_n\|^2 \\
& \leq & \E \| x_{n-1}-x\|^2 - 2 \gamma_n   \E \big[  f_n(x_{n-1}) - f_n(x) + \mu \|x_{n-1}-x\|^2 \big] \\
& & \hspace*{7cm} + 2\gamma_n \E \zeta_n^\top(x_{n-1}-x)+  \gamma_n^2 \E \|g_n\|^2\, .
  \EEAS
This leads to
\BEAS
\E f_n(x_{n-1})- f_n(x) 
& \leq & \E\| x_{n-1}-x\|^2(\frac{1}{2\gamma_n}-\frac{\mu}{2}) - \E\| x_{n}-x\|^2\frac{1}{2\gamma_n} +   \E \zeta_n^\top(x_{n-1}-x)+  \frac{\gamma_n}{2} \E \|g_n\|^2 .
\EEAS
First, we consider uniform averaging, induced by  the choice of $\gamma_n= \frac{1}{\mu n}$. Indeed, it gives
\BEAS
\E f_n(x_{n-1})- f_n(x) 
& \leq & \E\| x_{n-1}-x\|^2\frac{(n-1)\mu}{2}- \E\| x_{n}-x\|^2\frac{n\mu}{2} +   \E \zeta_n^\top(x_{n-1}-x)+  \frac{1}{2\mu n} \E \|g_n\|^2 .
\EEAS

Summing over $n$ and averaging gives
\BEAS
\frac{1}{N}\sum_{n=1}^N\E f_n(x_{n-1})- f_n(x) +  \|x_N-x\|^2\frac{N\mu}{2}  & \leq & \frac{1}{N}\sum_{n=1}^N \E \zeta_n^\top(x_{n-1}-x) +  \frac{1}{2N}\sum_{n=1}^N \frac{\E \|g_n\|^2}{\mu n}   \ .\EEAS

We now consider non-uniform averaging when $f_n=f$, induced by  the choice of $\gamma_n= \frac{2}{\mu (n+1)}$, which  gives
$$
\E f(x_{n-1})- f(x)
 \leq  \E\| x_{n-1}-x\|^2\frac{(n-1)\mu}{4}- \E\| x_{n}-x\|^2\frac{(n+1)\mu}{4} +   \E \zeta_n^\top(x_{n-1}-x)+  \frac{1}{\mu (n+1)} \E \|g_n\|^2 \\
$$
Multiplying by $n$, summing,  averaging and using the convexity of $f$ yield
$$
\E f(\hat{x}_{N-1})- f(x^\star) + \E\| x_{n}-x\|^2\frac{\mu}{2} \leq \frac{2}{N(N+1)}\sum_{n=1}^N\E \zeta_n^\top(x_{n-1}-x) +   \frac{1}{\mu (n+1)} \E \|g_n\|^2\, .
$$
\end{proof}

\section{Proof of  Propositions for Unconstrained Optimization}

\subsection{Proof of Proposition~\ref{prop:twoppoint-optim}}
\label{app:twoppoint-optim}

Our iteration is
$$
x_{n} = x_{n-1} - \gamma_n  \frac{d}{2 \delta_n}   \big[  f(x_{n-1}+\delta_n r_n u_n ) 
-  f(x_{n-1}- \delta_n  r_n u_n )  + \varepsilon_n \big] k(r_n)  u_n .
$$
We consider
\BEAS
g_n & = &   \frac{d}{2 \delta_n}   \big[  f(x_{n-1}+\delta_n r_n u_n ) 
-  f(x_{n-1}- \delta_n  r_n u_n )   \big] k(r_n)  u_n .
\EEAS

We will   need the expansion using the $\beta$-th order smoothness as:
\BEAS
f(x_{n-1}+\delta_n r_n u_n ) 
-  f(x_{n-1}- \delta_n  r_n u_n )   & \!\!\!\!\!= \!\!\!\!\! \!\!\!&  
\sum_{ |m| \leqslant \beta -1} \frac{1}{m!} f^{(m)}(x_{n-1}) \big[ ( \delta_n r_n )^m - ( -\delta_n r_n )^m \big]
+ [ A_n' - B_n' ],
\EEAS
with $|A_n'|, |B_n'| \leqslant \frac{M^{\beta}_\beta}{\beta !}   \delta_n ^\beta r_n^\beta$. When taking expectations above, we get exactly the term $2 \delta_n f'(x_{n-1})^\top u_n$.

Moreover, since $f$ is  2-smooth
 \BEAS
\big|  f(x_{n-1}+\delta_n r_n u_n ) 
-  f(x_{n-1}- \delta_n r_n u_n ) |
& \leq &   M_2^2 r_n^2  \delta_n^2 + 2 | f'(x_{n-1})^\top (  \delta r_n u_n) | \\
& \leq &   M_2^2 r_n^2  \delta_n^2 + 2 \delta_n r_n   |f'(x_{n-1})^\top u_n|.
 \EEAS

 We then get:
 \BEAS
 \E ( \| g_n \|^2 | \F_{n-1}) 
 & \leq & 
   \frac{d^2 \sigma^2 }{4 \delta^2}    \E_r  \big[ k(r)^2\big]
   + 
    \frac{d^2   }{4 \delta^2}
    2 M_2^4 \delta^4 \E [ r^4 k(r)^2 ]
    + \frac{d^2  }{4 \delta^2}
8 \delta^2 \E [ r^2 k(r)^2 ] \E \big[
 |f'(x_{n-1})^\top u_n|^2
 \big| \F_{n-1} \big]
\\
 & \leq & 
   \frac{3 \beta^3 d^2 \sigma^2 }{4 \delta_n^2}        + 
   4   {d^2 \beta^2   }     M_2^4 \delta_n^2  
    +     16 d \beta^2 \E \big[
 \| f'(x_{n-1})\|^2
\F_{n-1} \big] \mbox{ using } \E u_n u_n^\top = \frac{1}{d} \idm\, ,\\
& \leq &  \frac{3 \beta^3 d^2 \sigma^2 }{4 \delta_n^2}        + 
   4   {d^2 \beta^2   }     M_2^4 \delta_n^2   + 12 d M_2^2 \beta^2 \big[ f(x_{n-1}) - f(x_\ast) \big],
 \EEAS
where we used that $\| f'(x_{n-1})\|^2 \leq 2 M_2^2 \big[ f(x_{n-1}) - f(x_\ast) \big]$ for  $x_\ast$ a global optimizer of $f$.

Thus, 
\BEAS
& & \| x_n - x \|^2  \\
& = & \| x_{n-1} - x \|^2 - 2 \gamma_n  (x_{n-1}-x )^\top \big[
g_n  + \frac{d}{2 \delta_n}  \varepsilon_n   k(r_n)  u_n \big] + \gamma_n^2  \Big\| g_n  + 
\frac{d}{2 \delta_n}  \varepsilon_n   k(r_n)  u_n 
\Big\|^2
\\
& = & \| x_{n-1} - x \|^2 - 2 \gamma_n  (x_{n-1}-x )^\top \big[
g_n  + \frac{d}{2 \delta_n}  \varepsilon_n   k(r_n)  u_n \big]  + 2 \gamma_n^2   \| g_n   \|^2 + 2 \gamma_n^2 \Big\|
 \frac{d}{2 \delta_n}  \varepsilon_n   k(r_n)  u_n \Big\|^2.
\EEAS
By taking conditional expectations, we get, using $\E d r_n k(r_n)  u_n     u_n^\top  = I$, and the fact that the expectation of all powers $r_n^\alpha k(r_n)$, $\alpha>1$, lead to zero:
\BEAS
 && \E \big[ \| x_n - x \|^2  | \F_{n-1} \big] \\
& \leqslant & 
\| x_{n-1} - x \|^2 - 2 \gamma_n  (x_{n-1}-x )^\top  
 f'(x_{n-1}) +  2 \gamma_n   \E \| \frac{d}{2 \delta_n} [ A_n' - B_n'] k(r_n) u_n \|
\|  x_{n-1}-x \|
  \\
& & + 2 \gamma_n^2 \E ( \| g_n \|^2 | \F_{n-1}) + 
 2 \gamma_n^2 \E \Big\|
 \frac{d}{2 \delta_n}  \varepsilon_n   k(r_n)  u_n \Big\|^2
 \\
& \leqslant & 
\| x_{n-1} - x \|^2 - 2 \gamma_n \big[  f(x_{n-1}) - f(x ) \big] 
 +   \gamma_n   d\E \| \frac{1}{\delta_n}  \frac{M^{\beta}_\beta}{\beta !}  \delta_n ^\beta r_n^\beta k(r_n) u_n \|
\|  x_{n-1}-x \|
\\
& & 
+ 2 \gamma_n^2 \bigg[  \frac{3 \beta^3 d^2 \sigma^2 }{4 \delta_n^2}        + 
   4   {d^2 \beta^2   }     M_2^4 \delta_n^2   + 12 d M_2^2 \beta^2 \big[ f(x_{n-1}) - f(x_\ast) \big]
   \bigg]  + 
 2 \gamma_n^2 \big( \frac{d}{2 \delta_n} \big)^2 \sigma^2  
\E k(r_n)^2
\\
& \leqslant & 
\| x_{n-1} - x \|^2 - 2 \gamma_n \big[  f(x_{n-1}) - f(x ) \big] 
 +   \gamma_n  d \delta_n^{\beta-1}  \frac{M^{\beta}_\beta}{\beta !}  2\beta^2
\|  x_{n-1}-x \|
\\
& & 
+ 2 \gamma_n^2 \bigg[  \frac{3 \beta^3 d^2 \sigma^2 }{4 \delta_n^2}        + 
   4   {d^2 \beta^2   }     M_2^4 \delta_n^2   + 12 d M_2^2 \beta^2 \big[ f(x_{n-1}) - f(x_\ast) \big]
   \bigg]  + 
 6 \gamma_n^2 \big( \frac{d}{2 \delta_n} \big)^2 \sigma^2  
\beta^3.
\EEAS

For simplicity, we assume that $\gamma_n = \gamma$ is constant and less than $\frac{1}{24 d M_2^2 \beta^2}$, and that $\delta_n = \delta$. We thus get, with $x = x_\ast$:
\BEAS
\E  f(x_{n-1}) - f(x_\ast ) 
 & \leqslant & \frac{1}{\gamma} \E \| x_{n-1}-x_\ast \|^2 - \frac{1}{\gamma} \E \| x_{n}-x_\ast \|^2
  \\
  & &\hspace*{-1cm}  +2 \gamma \bigg[  \frac{3 \beta^3 d^2 \sigma^2 }{4 \delta^2}        + 
   4   {d^2 \beta^2   }     M_2^4 \delta^2   
   \bigg]  + 
 6 \gamma \big( \frac{d}{2 \delta} \big)^2 \sigma^2  
\beta^3
+
 d\delta^{\beta-1}  \frac{M^{\beta}_\beta}{\beta !}  2\beta^2
  \sqrt{ \E \| x_{n-1}-x_\ast \|^2 }.
\EEAS
Thus
\BEAS
\sum_{n=1}^N 
\E  \big[ f(x_{n-1}) - f(x )  \big]
+  \frac{1}{\gamma} \E \| x_{N}-x_\ast \|^2
& \leqslant & 
 \frac{1}{\gamma}  \| x_{0}-x_\ast \|^2 + 3 N \gamma d^2 \sigma^2 \delta^{-2} \beta^3
 + 8 N \gamma d^2 \beta^2 M_2^4 \delta^2 \\
 & & 
 + \sum_{n=1}^N 2d \delta^{\beta-1}  \frac{M^{\beta}_\beta}{\beta !}   \beta^2  \sqrt{ \E \| x_{n-1}-x_\ast \|^2 },
\EEAS
which we can put as:
$$
\sum_{n=1}^N  \big[
\E  f(x_{n-1}) - f(x_\ast ) \big]
+  \frac{1}{\gamma} \E \| x_{N}-x_\ast \|^2
\leqslant  \frac{1}{\gamma   }  \| x_{0}-x_\ast \|^2 + 
N  C
 + \sum_{n=1}^N 2d \delta^{\beta-1}  \frac{M^{\beta}_\beta}{\beta !}   \beta^2   \sqrt{ \E \| x_{n-1}-x_\ast \|^2 },
$$
with $C  =  3  \gamma d^2 \sigma^2 \delta^{-2} \beta^3
 + 8   \gamma d^2 \beta^2 M_2^4 \delta^2$.
This leads to, with $u_n =  \sqrt{ \E \| x_{n}-x_\ast \|^2 }$:
$$
u_N ^2 \leqslant u_0^2 + 
\gamma N   C
 + \sum_{n=1}^N 2 \gamma d \delta^{\beta-1}  \frac{M^{\beta}_\beta}{\beta !}   \beta^2 u_n.
$$
From Lemma 1 of~\citet{schmidt2011convergence}, we get:
\BEAS
u_N  & \leqslant & \frac{N}{2} 2\gamma d \delta^{\beta-1}  \frac{M^{\beta}_\beta}{\beta !}   \beta^2+ \bigg(
u_0^2 + 
\gamma N   C + \big[ \frac{N}{2}  2\gamma d \delta^{\beta-1}  \frac{M^{\beta}_\beta}{\beta !}   \beta^2 \big]^2
\bigg)^{1/2} \\
&
\leqslant &  N 2 \gamma d \delta^{\beta-1}  \frac{M^{\beta}_\beta}{\beta !}   \beta^2  + u_0  + ( \gamma N C)^{1/2}.
\EEAS
Thus
\BEAS
 & & \frac{1}{N} \sum_{n=1}^N 
\E  f(x_{n-1}) - f(x_\ast )  \\
&\leqslant  & 
 \frac{1}{\gamma N}  \| x_{0}-x_\ast \|^2 +   C
+  D
\bigg( 
N \gamma D  + u_0  + ( \gamma N C)^{1/2}
\bigg) 
\EEAS
with $D =  2 d \delta^{\beta-1}  \frac{M^{\beta}_\beta}{\beta !}   \beta^2 $.
 
By setting $\displaystyle \gamma = \frac{1}{24 d M_2^2 \beta^2  N^{(\beta+1)/(2\beta)}}$, and $ \displaystyle \delta = 
  \frac{\beta}{N^{1/(2\beta)} } (M_\beta^\beta M_2)^{-1/(\beta+1)}
  $, we get:
\BEAS
C & \leqslant &  \frac{d^2}{N^{(\beta+1)/(2\beta)} 24 d M_2^2 \beta^2 }
  \big[  3    \sigma^2   \beta^3 
   \frac{N^{1/\beta} }{\beta^2} (M_\beta^\beta M_2)^{2/(\beta+1)}
 + 8     \beta^2 M_2^4  
  \frac{\beta^2}{N^{1/\beta} } (M_\beta^\beta M_2)^{-2/(\beta+1)}
 \big] \\
 & \leqslant & 
 \frac{d   }{N^{(\beta-1)/(2\beta)} 24   \beta }
  (M_\beta/  M_2)^{2\beta/(\beta+1)}
  \big[  3    \sigma^2    
 + 8     
  \frac{\beta^3}{N^{2/\beta} } (M_\beta/M_2)^{-4\beta/(\beta+1)}
 \big] \\
  \frac{1}{\gamma N}  \| x_{0}-x_\ast \|^2 & \leqslant & 
  \frac{24 d   \beta^2 }{N^{(\beta-1)/(2\beta)} } ( M_2 \| x_{0}-x_\ast \| )^2 \\
  D  & \leqslant & 
  2d  \frac{M^{\beta}_\beta}{\beta !}   \beta^2 
    \frac{\beta^{\beta-1}}{N^{(\beta-1)/(2\beta)} } (M_\beta^\beta M_2)^{-(\beta-1)/(\beta+1)}
    \\
     & \leqslant & 
  2d    
    \frac{\beta}{N^{(\beta-1)/(2\beta)} } (    M_\beta / M_2 )^{(\beta-1)/(\beta+1)}M_\beta.
\EEAS
This leads to an overall rate of
\BEAS
 & & \frac{1}{N} \sum_{n=1}^N 
\E  f(x_{n-1}) - f(x )  \\
&\leqslant  &   
2 D^2 \gamma N  +  \frac{2}{\gamma N}  \| x_{0}-x_\ast \|^2 + 2 C  \\
&\leqslant  &   
2 \bigg( 2d    
    \frac{\beta}{N^{(\beta-1)/(2\beta)} } (    M_\beta / M_2 )^{(\beta-1)/(\beta+1)}M_\beta \bigg)^2   \frac{1}{24 d M_2^2 \beta^2  N^{(\beta+1)/(2\beta)}} N  + \frac{48 d   \beta^2 }{N^{(\beta-1)/(2\beta)} } ( M_2 \| x_{0}-x_\ast \| )^2  
 \\
 & & +  \frac{2d   }{N^{(\beta-1)/(2\beta)} 24   \beta }
  (M_\beta/  M_2)^{2\beta/(\beta+1)}
  \big[  3    \sigma^2    
 + 8     
  \frac{\beta^3}{N^{2/\beta} } (M_\beta/M_2)^{-4\beta/(\beta+1)}
 \big]  \\
 & \leqslant & 
 \frac{d  }{N^{(\beta-1)/(2\beta)}}
 \bigg( 48   \beta^2  ( M_2 \| x_{0}-x_\ast \| )^2  
 + 6 \sigma^2   (M_\beta/  M_2)^{2\beta/(\beta+1)}
 + \frac{1}{3}  (    M_\beta / M_2 )^{4 \beta /(\beta+1)}
 \bigg)\\
 & & \frac{16 d   }{N^{(\beta-1)/(2\beta)} 24     }
  \frac{\beta^2}{N^{2/\beta} } (M_\beta/M_2)^{-2\beta/(\beta+1)}
\\
 & \leqslant & 
 \frac{d  }{N^{(\beta-1)/(2\beta)}}
 \bigg( 7   \beta M_2 \| x_{0}-x_\ast \|  
 + 3 \sigma +    (M_\beta/  M_2)^{2\beta/(\beta+1)}
 +   \frac{\beta}{N^{1/\beta} } (M_\beta/M_2)^{-\beta/(\beta+1)}
 \bigg)^2,
 \EEAS
 which is almost the desired bound, except the dependence on $d$, which is in $d$ instead of $d^{(\beta-1)/\beta}$.
 Like in the proof for constrained optimization, we can choose $\gamma$ and $\delta$ with slightly different scalings in $d$, that is, 
 $\displaystyle \gamma = \frac{1}{24 d^{(\beta-1)/\beta} M_2^2 \beta^2  N^{(\beta+1)/(2\beta)}}$, and $ \displaystyle \delta = 
  \frac{\beta d^{1/\beta}}{N^{1/(2\beta)} } (M_\beta^\beta M_2)^{-1/(\beta+1)}
  $. The value of $\gamma$ does not satisfy our constraint when $d^{ -1/\beta}    N^{(\beta+1)/(2\beta)}$ is less than one, which happens only when the final bound is trivial. Thus, we can safely consider the step-size $\gamma$ above.

\paragraph{Proof for anytime algorithm}
By setting $\displaystyle \gamma_n = \frac{1}{24 d M_2^2 \beta^2  n^{(\beta+1)/(2\beta)}}$, and $ \displaystyle \delta_n = 
  \frac{\beta}{n^{1/(2\beta)} } (M_\beta^\beta M_2)^{-1/(\beta+1)}
  $, as a function of $n$, we obtain an anytime algorithm. In order to analyze it, we can simply recycle the proof techniques
  of \citet{gradsto} (in particular Abel's summation formula). All sums of the forms  $\sum_{n=1}^N n^{-\delta}$ may then be bounded thrtough $ \frac{N^{1-\delta}}{1-\delta}$ for $\delta \in (0,1)$ and less than $\frac{1}{\delta-1}$ for $\delta>1$, with $\sum_{n=1}^N \frac{1}{n} \leq \log (N+1)$.
  The term $\gamma_n^2 \delta_n^{-2}$   leads to an extra factor of 
  $\log(N+1)$ while all other factors only lead to extra \emph{constant} factors which are less than $4$. The final bound is thus the same as before up to logarithmic terms

 \subsection{Proof of Proposition~\ref{prop:twoppoint-optim-strongly}}
\label{app:twoppoint-optimstrongly}
\label{app:twoppoint-optim-strongly}
The proof technique is the same as for  Proposition~\ref{prop:twoppoint-optim}
in Appendix~\ref{app:twoppoint-optim}.
The first line that differs is the following, where $\mu$-strong convexity is used:
 \BEAS
& & \E \big[ \| x_n - x_\ast\|^2  | \F_{n-1} \big] \\
& \leqslant & 
( 1 - \mu \gamma_n )\| x_{n-1} -  x_\ast \|^2 - 2 \gamma_n \big[  f(x_{n-1}) - f( x_\ast ) \big] 
 +   \gamma_n  d \delta_n^{\beta-1}  \frac{M^{\beta}_\beta}{\beta !}  2\beta^2
\|  x_{n-1}-x \|
\\
& & 
+ 2 \gamma_n^2 \bigg[  \frac{3 \beta^3 d^2 \sigma^2 }{4 \delta_n^2}        + 
   4   {d^2 \beta^2   }     M_2^4 \delta_n^2   + 12 d M_2^2 \beta^2 \big[ f(x_{n-1}) - f(x_\ast) \big]
   \bigg]  + 
 6 \gamma_n^2 \big( \frac{d}{2 \delta_n} \big)^2 \sigma^2  
\beta^3.
\EEAS
If we assume that  $\gamma_n$ is less than $\frac{1}{24 d M_2^2 \beta^2}$, then
we get
\BEA
\label{eq:EEE}
\E  f(x_{n-1}) - f(x_\ast ) 
 & \leqslant & ( \frac{1}{\gamma_n} - \mu) \E \| x_{n-1}-x_\ast \|^2 - \frac{1}{\gamma_n} \E \| x_{n}-x_\ast \|^2
  \\
\nonumber  & & +2 \gamma_n \bigg[  \frac{3 \beta^3 d^2 \sigma^2 }{2 \delta_n^2}        + 
   4   {d^2 \beta^2   }     M_2^4 \delta_n^2   
   \bigg]  + 
 d\delta_n^{\beta-1}  \frac{M^{\beta}_\beta}{\beta !}  2\beta^2
  \sqrt{ \E \| x_{n-1}-x_\ast \|^2 }.
\EEA
In order to bound $\sqrt{ \E \| x_{n-1}-x_\ast \|^2 }$, we use the same proof technique than in Appendix~\ref{app:twoppoint-optim}, without using strong convexity and from the equation:
\BEAS
\E \| x_{n}-x_\ast \|^2
 & \leqslant &  \E \| x_{n-1}-x_\ast \|^2
  \\
  & & +2 \gamma_n^2 \bigg[  \frac{3 \beta^3 d^2 \sigma^2 }{2 \delta_n^2}        + 
   4   {d^2 \beta^2   }     M_2^4 \delta_n^2   
   \bigg]  
+
 \gamma_n d\delta_n^{\beta-1}  \frac{M^{\beta}_\beta}{\beta !}  2\beta^2
  \sqrt{ \E \| x_{n-1}-x_\ast \|^2 },
\EEAS
which leads to
\BEAS
\E \| x_{n}-x_\ast \|^2
 & \leqslant &  \E \| x_{0}-x_\ast \|^2
  \\
  & & +2 \sum_{k=1}^n \gamma_k^2 \bigg[  \frac{3 \beta^3 d^2 \sigma^2 }{4 \delta_k^2}        + 
   4   {d^2 \beta^2   }     M_2^4 \delta_k^2   
   \bigg]+ \sum_{k=1}^n
 \gamma_k \delta_k^{\beta-1}  d \frac{M^{\beta}_\beta}{\beta !}  2\beta^2
  \sqrt{ \E \| x_{k-1}-x_\ast \|^2 }\\
   & \leqslant &  \E \| x_{0}-x_\ast \|^2
  \\
  & & + B+ \sum_{k=1}^n
 \gamma_k \delta_k^{\beta-1}  d \frac{M^{\beta}_\beta}{\beta !}  2\beta^2
  \sqrt{ \E \| x_{k-1}-x_\ast \|^2 },
\EEAS
with $B = 2 \sum_{k=1}^n \gamma_k^2 \bigg[  \frac{3 \beta^3 d^2 \sigma^2 }{4 \delta_k^2}        + 
   4   {d^2 \beta^2   }     M_2^4 \delta_k^2   
   \bigg]$.
   
Thus, with $u_n =  \sqrt{ \E \| x_{n}-x_\ast \|^2 }$, we have:
$$
u_n ^2 \leqslant u_0^2 + 
B
 + \sum_{k=1}^n
 \gamma_k \delta_k^{\beta-1}  d \frac{M^{\beta}_\beta}{\beta !}  2\beta^2
u_k
$$
From Lemma 1 of~\citet{schmidt2011convergence}, we get:
\BEAS
u_n  & \leqslant &  \sum_{k=1}^n
 \gamma_k \delta_k^{\beta-1}  d \frac{M^{\beta}_\beta}{\beta !}  2\beta^2 + u_0 + B^{1/2}.
\EEAS

We now choose $\gamma_n =\frac{1}{ n \mu}$, which is less than
$\frac{1}{24 d M_2^2 \beta^2}$ only for certain values of $n$ (if this is not satisfied, the bound is trivial anyway, so this restriction does not impact the result). We select $\delta_n = \bigg( \frac{d^2 \beta!}{ M_\beta^\beta \mu n} \bigg)^{1/(\beta+1)}$.

Then, we may follow the previous proof and sum \eq{EEE}, with telescoping elements and the same formulas (except the leading terms in $n\mu$, leading to the following bound:
 \BEAS
 & & \frac{1}{N} \sum_{n=1}^N 
\E  f(x_{n-1}) - f(x )  \\
 & \leqslant & 
    \big( \frac{d^2 M_\beta^2}{n \mu } \big) ^{(\beta-1) / ( \beta+1 )}
 \bigg( 8   \beta M_\beta \| x_{0}-x_\ast \|  
 + 4 \sigma +   2  +    {\beta }  (M_2/M_\beta)^{2}  \big( \frac{M_\beta^2}{n \mu } \big) ^{2 / ( \beta+1 )}
 \bigg)^2.
 \EEAS

\section{Proof of  Propositions in Constrained Optimization}
\subsection{Proof of Proposition~\ref{prop:onepoint-nonstrong}}
\label{app:onepoint-nonstrong}

We recall that the gradient estimate is $g_n = \frac{d}{\delta_n}\Big(f(x_{n-1}+\delta_nr_nu_n)+\varepsilon_n\Big)k(r_n)u_n$, so that
$$ \E g_n = \hat{f}_{\delta_n}'(x_{n-1})= f'_n(x_{n-1}) + \zeta_n,\  \text{ with } \ \|\zeta_n\| \leq 2\sqrt{2}\frac{M_\beta^\beta\beta}{(\beta-1)!}\delta_n^{\beta-1}\, .$$
and the variance of $g_n$ is bounded as
$$ \E \|g_n\|^2 \leq 6\beta^3\frac{d^2}{\delta^2_n}(C_{\delta_n}^2+\sigma^2) \leq 6\beta^3\frac{d^2}{\delta^2_n}(C_{\delta_1}^2+\sigma^2) $$
 Using Proposition \ref{PR:SGD}, along with the specific choices of $$\gamma_n = \frac{R \delta_n }{\sqrt{ \beta^3}d \sqrt{n}} \ \text{  and  } \  
\displaystyle \delta_n^{\beta } = 
\frac{    d   \sqrt{\beta}(\beta-1) ! }{
\sqrt{n} M^\beta_\beta },$$   lead to
\BEAS\frac{1}{N}\sum_{n=1}^N 
\E \big[  f_n(x_{n-1}) - f_n(x) \big]
&\leq& \frac{R^2}{2\gamma_NN}+  3\beta^3d^2 \frac{1}{N} \sum_{n=1}^N\frac{\gamma_n}{\delta_n^2}  ( C_{\delta_n}^2 + \sigma^2)  
+ 2\sqrt{2}\frac{M^\beta_\beta}{(\beta-1) !}\beta R\frac{1}{N}\sum_{n=1}^N \delta_n^{\beta-1}\\
&\leq& \frac{R^2}{2\gamma_NN}+\frac{4RM_\beta}{N}\sum_{n=1}^N\bigg( \frac{  d \sqrt{\beta} }{ \sqrt{n} }  \bigg)^{\frac{\beta-1}{\beta}} (3C_{\delta_n}^2 + 3\sigma^2+ 2\sqrt{2})\\
&\leq& 25R M_\beta \bigg(\frac{ d^2 \beta}{ N}\bigg)^{\frac{\beta-1}{2\beta}}(C^2_{\delta_1}+\sigma^2+1)\, .
\EEAS

 \subsection{Proof of Proposition~\ref{prop:oneppoint-online-strongly}}
\label{app:oneppoint-onlinestrongly}

Using the same bounds on the biais and variance of $g_n$ than in the proof of Proposition~\ref{prop:onepoint-nonstrong} along with the results of Proposition \ref{PR:SGDstrong} give  

 \BEAS
\frac{1}{N}\sum_{n=1}^N 
\E \big[ f_n(x_{n-1}) - f_n(x) \big] + \frac{\mu}{2}\E\| x_{N}-x\|^2
&\leq&   3 \beta^3   \frac{d^2}{\mu}   \frac{1}{N}\sum_{n=1}^N \frac{1}{n\delta_n^2}  ( C_{\delta_n}^2 + \sigma^2)  \\
& &+\frac{2\sqrt{2}\beta M_\beta^\beta}{(\beta-1)!}R\frac{1}{N}\sum_{n=1}^N\delta_n^{\beta-1}
\EEAS

 The specific choice of $\delta_n^{\beta+1}=\frac{\beta  d^2\beta!}{n\mu M^\beta_\beta  } $ ensures that the upper bound is smaller than
$$ 15\beta^2M_\beta^{\frac{2\beta}{\beta+1}}\bigg(\frac{d^2}{\mu N}\bigg)^{\frac{\beta-1}{\beta+1}}(C_{\delta_1}+\sigma^2+1)\, .
$$

  \subsection{Proof of Proposition~\ref{PR:beta2}}
 \label{app:8}
Proposition 7 is another consequence of Propositions \ref{PR:SGD} and \ref{PR:SGDstrong}. Indeed, for $\beta=2$, the mapping $\hat{f}_\delta$ is convex, hence we can consider the algorithms as stochastic gradient descents on  $\hat{f}_\delta$, with an unbiased estimate of the gradient. Then it suffices to approximate $\hat{f}_\delta(x)$ by $f(x) \pm 2\sqrt{2}\beta\frac{M_\beta^\beta}{\beta !}\delta^\beta$, to choose parameters so that error terms balance and  to conclude.

 \subsection{Proof of Proposition~\ref{prop:oneppoint-optim-strongly}}

\label{app:oneppoint-optimstrongly}

Once again, the proof uses the same standard arguments than the  proof of Proposition \ref{PR:SGDstrong}. More precisely, we consider here constant step size $\delta_n=\delta$, where $\delta$ is small enough so that $\hat{f}_\delta$ is $\mu'$-strongly convex (where $\mu' \leq \mu$ and, as we will see, it will be implied by $N$ being big enough) and we apply Proposition \ref{PR:SGDstrong} to $\hat{f}_\delta$, this allows us to bound $\E\| x_{N}-x^\sharp\|^2$, where $x^\sharp$ is a minimizer of $\hat{f}_\delta$  

Finally, we conclude using  the smoothness and the strong convexity of $f$ that imply that 
 $$
\|x^\sharp-x^\star \| \leq \frac{1}{\mu'} \|f'(x^\sharp)\| \leq \frac{1}{\mu'} \|f'(x^\sharp)-f_\delta'(x^\sharp)\|\leq \frac{1}{\mu'} \frac{M^{\beta}_{\beta}}{(\beta-1)!}\delta^{\beta-1}2\sqrt{2}\beta.$$
As a consequence the triangle inequality
\BEAS
\E\| x_{N}-x^\star\|^2 &\leq & 2\E\| x_{N}-x^\sharp\|^2 +2\|x^\sharp-x^\star \| ^2
\EEAS
and the combined above majorations of  $\E\| x_{N}-x^\sharp\|^2$ and $2\|x^\sharp-x^\star \| ^2$ give the result.

We emphasize agains that the fact that $f_\delta$ is $\mu'$-strongly convex is  ensured by $N$ being large enough (and the larger $N$, the bigger $\mu' \leq \mu$ can be chosen).

 \end{document}